\documentclass[letterpaper]{article}
\usepackage{aaai}
\usepackage{times}
\usepackage{helvet}
\usepackage{courier}
\usepackage{amsmath}
\usepackage{amssymb}
\usepackage{listings}
\usepackage{graphicx}
\usepackage{url}
\usepackage{color}

\usepackage{amsmath}
\usepackage{graphicx}

\newcommand\numberthis{\addtocounter{equation}{1}\tag{\theequation}}

\newcommand{\cblu}{\color{black}}

\newcommand{\cbla}{\color{black}}

\long\def\BOC#1\EOC{\message{(Commented text )}}
\long\def\BOCC#1\EOCC{\message{(Commented text )}}
\long\def\BOCCC#1\EOCCC{\message{(Commented text )}}
\long\def\optional#1{\empty}
\long\def\NB#1{}
\long\def\NBB#1{}

\def\o{\overline}
\def\ar{\leftarrow}
\def\bi{\begin{itemize}}
\def\ei{\end{itemize}}
\def\beq{\begin{equation}}
\def\eeq#1{\label{#1}\end{equation}}
\def\ba{\begin{array}}
\def\ea{\end{array}}

\def\mi#1{\mathit{#1}}

\def\ar{\leftarrow}
\def\rar{\rightarrow}

\def\false{\hbox{\sc false}}
\def\true{\hbox{\sc true}}

\def\qed{\quad \vrule height7.5pt width4.17pt depth0pt \medskip}

\DeclareSymbolFont{AMSa}{U}{msa}{m}{n}
\DeclareMathSymbol{\square}{\mathord}{AMSa}{"03}

\def\mu#1{\mathit{\underline{#1}}}

\def\fand{\otimes}
\def\for{\oplus}
\def\fneg{\neg}
\def\frar{\rar}

\def\bi{\begin{itemize}}
\def\ei{\end{itemize}}

\def\u{\upsilon}

\newtheorem{thm}{Theorem}
\newtheorem{cor}{Corollary}
\newtheorem{definition}{Definition}
\newtheorem{lemma}{Lemma} 
\newtheorem{example}{Example}

\def\L{\mathbb{L}}

\nocopyright

\title{On the Semantic Relationship between Probabilistic Soft Logic and Markov Logic}

\author{Joohyung Lee \and Yi Wang\\
School of Computing, Informatics, and Decision Systems Engineering \\
Arizona State University\\
Tempe, AZ, 85287, USA \\
{\tt \{joolee,ywang485\}@asu.edu}}

\begin{document}
\maketitle

\begin{abstract}
Markov Logic Networks (MLN) and Probabilistic Soft Logic (PSL) are widely applied formalisms in Statistical Relational Learning, an emerging area in Artificial Intelligence that is concerned with combining logical and statistical AI. Despite their resemblance, the relationship has not been formally stated. In this paper, we describe the precise semantic relationship between them from a logical perspective. This is facilitated by first extending fuzzy logic to allow weights, which can be also viewed as a generalization of PSL, and then relate that generalization to MLN.  We observe that the relationship between PSL and MLN is analogous to the known relationship between fuzzy logic and Boolean logic, and furthermore the weight scheme of PSL is essentially a generalization of the weight scheme of MLN for the many-valued setting.\footnote{%
In Working Notes of the 6th International Workshop on Statistical Relational AI (StarAI 2016)
}
\end{abstract}


\section{Introduction}

Statistical relational learning (SRL) is an emerging area in Artificial Intelligence that is concerned with combining logical and statistical AI.
Markov Logic Networks (MLN) \cite{richardson06markov} and Probabilistic Soft Logic (PSL) \cite{kimmig12ashort,bach15hinge} are well-known formalisms in statistical relational learning, and have been successfully applied to a wide range of AI applications, such as natural language processing, entity resolution, collective classification, and social network modeling. Both of them combine logic and probabilistic graphical model in a single representation, where each formula is associated with a weight, and the probability distribution over possible worlds is derived from the weights of the formulas that are satisfied by the possible worlds. However, despite their resemblance to each other, the precise relationship between their semantics is not obvious. 
PSL is based on fuzzy interpretations that range over reals in $[0,1]$, and in this sense is more general than MLN. On the other hand, its syntax is restricted to formulas in clausal form, unlike MLN that allows any complex formulas. It is also not obvious how their models' weights are related to each other due to the different ways that the weights are associated with models. Originating from the machine learning research, these formalisms are equipped with several efficient inference and learning  algorithms, and some paper compares the suitability of one formalism over the other by experiments on specific applications \cite{beltagy14probabilistic}. 
On the other hand, the precise relationship between the two formalisms has not been formally stated. 


In this paper, we present a precise {\em semantic} relationship between them. We observe that the relationship is analogous to the well-known relationship between fuzzy logic and classical logic. Moreover, despite the different ways that weights of models are defined in each formalism, it turns out that they are essentially of the same kind. Towards this end, we introduce a weighted fuzzy logic as a proper generalization of PSL, which is also interesting on its own as an extension of the standard fuzzy logic to incorporate weighted models. The weighted fuzzy logic uses the same weight scheme as PSL, but associates weights to arbitrary fuzzy formulas. This intermediate formalism facilitates the comparison between PSL and MLN. We observe that the same analogy between fuzzy logic and Boolean logic carries over to between PSL and MLN. Analogous to that fuzzy logic agrees with Boolean logic on {\em crisp} interpretations, PSL and MLN agree on crisp interpretations, where their weights are proportional to each other. However, their maximum a posteriori (MAP) estimates do not necessarily coincide due to the differences between many-valued vs. Boolean models.

The paper is organized as follows. 
We first review each of MLN, fuzzy propositional logic, and PSL. Then we define a weighted fuzzy logic as a generalization of PSL. Using this we study the semantic relationship between PSL and MLN.


\section{Preliminaries} 

Although both PSL and MLN allow atoms to contain variables, those variables are understood in terms of grounding over finite domains where a universally quantified sentence is turned into multiple conjunctions and an existentially quantified sentence is turned into multiple disjunctions,  essentially resulting in propositional theories. For example, the ground atoms of the first-order signature $\sigma = \{p, a, b\}$, where $p$ is a unary predicate constant and $a$, $b$ are object constants, can be identified with the propositional atoms of the propositional signature $\{p(a), p(b)\}$. Thus for simplicity but without losing generality, we assume that the programs are propositional.\footnote{Inference and learning algorithms in these languages indeed utilize the relational structure, but in terms of defining the semantics, the assumption simplifies the presentation without the need to refer to fuzzy {\sl predicate} logic.}

\subsection{Review: Markov Logic Networks} \label{ssec:review-mln}


The following is a review of Markov Logic from~\cite{richardson06markov}.
A \emph{Markov Logic Network (MLN)}~$\L$ of a propositional signature $\sigma$ is a finite set of pairs \hbox{$\langle w: F\rangle$}, where $F$ is a propositional formula of~$\sigma$ and $w$ is a real number.

For any MLN $\L$ of signature $\sigma$, we define $\L_I$ to be the set of weighted formulas $w:F$ in~$\L$ such that $I\models F$. The {\em unnormalized weight} of an interpretation $I$ under~$\L$ is defined as 
\[ 
  W_\L(I) = exp\Bigg(\sum_{w:F \in \L_I} w\Bigg), 
\] 
and the {\em normalized weight (a.k.a. probability)} of $I$ under $\L$  is defined as
\[ 
  P_\L(I) = \frac{W_\L(I)}{\sum_{J\in PW}{W_\L(J)}}, 
\] 
where $PW$ (``Possible Worlds'') is the set of all interpretations of~$\sigma$. 

The basic idea of Markov Logic is to allow formulas to be {\em soft constrained}, where a model does not have to satisfy all formulas, but is associated with the weight that is obtained from the satisfied formulas.
An interpretation that does not satisfy certain formulas receives an ``(indirect) penalty'' because such formulas do not contribute to the weight of that interpretation. 


\subsection{Review: Fuzzy Propositional Formula}

The following is a review of fuzzy propositional formulas from~\cite{hajek98mathematics}. 
A fuzzy propositional signature~$\sigma$ is a set of symbols called {\em fuzzy atoms}. In addition, we assume the presence of a set ${\rm CONJ}$ of fuzzy conjunction symbols, a set ${\rm DISJ}$ of fuzzy disjunction symbols, a set ${\rm NEG}$ of fuzzy negation symbols, and a set ${\rm IMPL}$ of fuzzy implication symbols. 

A {\em fuzzy (propositional) formula} of $\sigma$ is defined recursively as follows.
\begin{itemize}
\item every fuzzy atom $p\in\sigma$ is a fuzzy formula;

\item every numeric constant $c$, where $c$ is a real number in $[0, 1]$, is a fuzzy formula;

\item if $F$ is a fuzzy formula, then $\fneg F$ is a fuzzy formula, where $\fneg \in
  {\rm NEG}$;

\item if $F$ and $G$ are fuzzy formulas, then $F\fand G$, $F\for G$, 
  and $F\frar G$ are fuzzy formulas, where $\fand\in{\rm CONJ}$,
  $\for\in{\rm DISJ}$, and $\frar\ \in{\rm IMPL}$.
\end{itemize}

The models of a fuzzy formula are defined as follows.
The {\em fuzzy truth values} are the real numbers in the range $[0,1]$. 
A \emph{fuzzy interpretation} $I$ of~$\sigma$ is a mapping from $\sigma$ into $[0, 1]$.

The fuzzy operators are functions mapping one or a pair of truth values into a truth value. Among the operators, $\fneg$ denotes a function from $[0, 1]$ into $[0, 1]$; $\fand$, $\for$, and $\frar$ denote functions from $[0, 1]\times[0, 1]$ into $[0, 1]$.
The actual mapping performed by each operator can be defined in many different ways, but all of them satisfy the properties that they are generalizations of the corresponding Boolean connectives.
\BOC
following conditions, which imply that the operators are generalizations of the corresponding classical propositional connectives:\footnote{%
We say that a function $f$ of arity $n$ is \emph{increasing in its $i$-th argument} ($1\leq i\leq n$) if $f(arg_1,\dots, arg_i,\dots, arg_n)\leq f(arg_1,\dots,arg_i^\prime,\dots, arg_n)$ for all arguments such that $arg_i \leq arg_i^\prime$; 
$f$ is said to be \emph{increasing} if it is increasing in all its arguments. The definition of \emph{decreasing} is similarly defined.  
}

\begin{itemize}
\item a fuzzy negation $\fneg$ is decreasing, and satisfies $\neg(0) = 1$ and $\neg(1) = 0$;

\item a fuzzy conjunction $\fand$ is increasing, commutative, associative, and $\fand(1, x)=x$ for all $x \in [0, 1]$;

\item a fuzzy disjunction $\for$ is increasing, commutative, associative, and $\for(0, x)=x$ for all $x\in [0, 1]$;

\item a fuzzy implication $\frar$ is decreasing in its first argument and increasing in its second argument; and $\frar\!\!(1, x) = x$ and $\frar\!\!(0, 0) = 1$ for all $x\in[0, 1]$.
\end{itemize}
\EOC
Figure~\ref{fig:operators} lists some examples of fuzzy operators.

\begin{figure}
\begin{center}
{\scriptsize 
\begin{tabular}{| c | l | l |}
			\hline
			\textbf{Symbol} & \textbf{Name} & \textbf{Definition} \\
			\hline
			$\fand_l$ & Lukasiewicz t-norm & $\fand_l(x, y)=max(x + y -1, 0)$ \\
			$\for_l$ & Lukasiewicz t-conorm & $\for_l(x, y)=min(x + y, 1)$ \\
			\hline
			$\fand_m$ & G\"odel t-norm & $\fand_m(x, y)=min(x, y)$ \\
			$\for_m$ & G\"odel t-conorm & $\for_m(x, y)=max(x, y)$ \\
			\hline
			$\fand_p$ & product t-norm & $\fand_p(x, y)=x \cdot y$ \\
			$\for_p$ & product t-conorm & $\for_p(x, y)=x+ y -x\cdot y$ \\
			\hline
			$\neg_s$ & standard negator & $\neg_s(x) =
                        1-x$\\ \hline

			$\rar_r$ & R-implicator induced by $\fand_m$ &
                        $\rar_r\!\!(x, y)=\begin{cases}1 & \text{if}\ x \leq y\\y
                          & \text{otherwise}\end{cases}$\\
			$\rar_s$ & S-implicator induced by $\fand_m$
                        &  $\rar_s\!\!(x, y) =
                        max(1-x, y)$\\  
                        $\rar_l$ & Implicator induced by $\fand_l$ & $\rar_l\!\!(x,y) = min(1-x+y, 1)$ \\
\hline
\end{tabular}
\caption{Some t-norms, t-conorms, negator, and implicators}
\label{fig:operators}
}
\end{center}
\end{figure}

The \emph{truth value} of a fuzzy propositional formula $F$ under~$I$, denoted $\u_I(F)$, is defined recursively as follows:  
\begin{itemize}
\item  for any atom $p\in\sigma$,\ \ $\u_I(p) = I(p)$;
\item  for any numeric constant ${c}$,\ \  $\u_I({c}) = c$;
\item  $\u_I(\neg F) = \neg(\u_I(F))$;
\item  $\u_I(F\odot G) = \odot(\u_I(F), \u_I(G))$ \ \ \ ($\odot\in\{\fand, \for, \frar\}$).
\end{itemize}
(For simplicity, we identify the symbols for the fuzzy operators with the truth value functions represented by them.) 

\begin{definition}\label{def:fuzzy-m}\optional{def:fuzzy-m}
We say that a fuzzy interpretation $I$ {\em satisfies} a fuzzy formula~$F$ if \hbox{$\u_I(F) = 1$}, and denote it by $I\models F$. We call such $I$ a {\em fuzzy model} of~$F$.
\end{definition}

We say that a fuzzy interpretation $I$ is {\em Boolean} if $I(p)$ is either $0$ or $1$ for each fuzzy atom $p$. Clearly, we may identify a Boolean fuzzy interpretation $I$ with the classical propositional interpretation by identifying $1$ with $\true$ and $0$ with $\false$.

Any fuzzy propositional formula whose numeric constants are restricted to $0$ and $1$ can be identified with a classical propositional formula. For such a formula $F$, due to the fact that fuzzy operators are generalizations of their Boolean counterparts, it is clear that Boolean fuzzy models of $F$ are precisely the Boolean models of $F$ when $F$ is viewed as a classical propositional formula. 


\subsection{Review: Probabilistic Soft Logic}

The following is a review of PSL from~\cite{kimmig12ashort}, but is stated using the terminology from fuzzy logic.
A {\em PSL program}~$\Pi$ is a set of weighted formulas 
$\langle w: R\ \verb!^!  k\rangle$ where 
\begin{itemize}
\item  $w$ is a nonnegative real number,
\item  $R$ is a fuzzy propositional formula of the form \footnote{We understand 
$G\ar F$ as an alternative notation for $F\rar G$.}
       \beq
          a \ar_l b_1 \fand_l \dots \fand_l b_n  
      \eeq{eq:psl-rule}
      where $n\ge 0$, each of $a, b_1, \dots, b_n$ is a fuzzy atom possibly preceded by the standard negator, and 
\item $k\in\{1,2\}$.\footnote{PSL also allows linear equality and inequality constraints, which is outside logical theories, and we omit here for simplicity. Interpretation $I$ that violates any of them gets $f_\Pi(I) =0$.}
\end{itemize}
For each rule $R$ of the form~\eqref{eq:psl-rule}, the {\em distance to satisfaction} under interpretation $I$ is defined as 
\beq
   d_R(I) = max\{0,\ \  \u_I(b_1\fand_l\cdots\fand_l b_n)- \u_I(a)\}.
\eeq{distance}
Given an interpretation $I$ of $\Pi$, the {\em unnormalized density function} over $I$ under $\Pi$ is defined as
\[
   \hat{f}_\Pi(I) = 
   exp\Bigg(-\sum_{\langle w: R \string^ k\rangle \in \Pi} w\cdot d_R(I)^k \Bigg), 
\]
and the {\em probability density function} over $I$ under $\Pi$ is defined as
\[
   f_\Pi(I) = \frac{\hat{f}_\Pi(I)}{Z_\Pi}, 
\]
where $Z_\Pi$ is the normalization factor 
\[
   \int_I \hat{f}_\Pi(I) .
\]

The probability density function $f_\Pi(I)$ is defined similar to the weight $W_\L(I)$ in MLN. Different from MLN where the weight of an interpretation comes from the sum over the weights of all formulas that are satisfied (thus the penalty is implicit), in PSL, the probability density function of an interpretation is obtained from the sum over the ``penalty'' (i.e., the weight times the distance to satisfaction) from each formula, where the penalty is $0$ when the formula is satisfied, and becomes bigger as the formula gets unsatisfied more (i.e., the fuzzy truth value of the body gets bigger than the fuzzy truth value of the head). When the formula is most unsatisfied (i.e., the body evaluates to $1$ and the head evaluates to $0$), the penalty is $w$, the maximum. A novel idea here is that each formula contributes to the penalty to a certain graded truth degree (including~$0$). 
Along with the restriction imposed on the syntax of fuzzy formulas (using the rule form~\eqref{eq:psl-rule}), MAP inference in PSL can be reduced to a convex optimization problem in continuous space, thereby enabling efficient computation.




\section{Weighted Fuzzy Logic as a Generalization of PSL}

\subsection{Weighted Fuzzy Logic}

Here we define a weighted fuzzy logic as a generalization of PSL. The idea is simple. We take the standard fuzzy logic and extend it by applying the log-linear weight scheme of PSL. 


A {\em weighted propositional fuzzy logic theory} $\Pi$ is a set of weighted formulas\hbox{$\langle w: F\verb!^! k\rangle$}, where
\begin{itemize}
\item  $w$ is a real number,
\item  $F$ is a fuzzy propositional formula, and 
\item  $k\in\{1,2\}$.
\end{itemize}

The {\em unnormalized density function} of a fuzzy interpretation $I$ under $\Pi$ is defined as
\[
  \hat{f}_{\Pi}(I) = exp\Bigg(-\sum_{\langle w: F \string^ k\rangle \in \Pi} w\cdot (1-\u_I(F))^k\Bigg), 
\]
and the {\em probability density function} of $I$ under $\Pi$ is defined as 
\[
  f_{\Pi}(I) = \frac{\hat{f}_\Pi(I)}{Z_\Pi},
\]
where $Z_\Pi$ is the normalization factor 
\[ 
  \int_I \hat{f}_\Pi(I).
\]


Notice that $1-\u_I(F)$ represents the distance to satisfaction in the general case. It is $0$ when $I$ satisfies $F$, and becomes bigger as $\u_I(F)$ gets farther from $1$. {\cblu This notion of distance to satisfaction for an arbitrary formula is also used in Probabilistic Similarity Logic \cite{brocheler10probabilistic}, and indeed, the weighted fuzzy logic is very similar to Probabilistic Similarity Logic. Both of them employ arbitrary fuzzy operators, not restricted to the Lukasiewicz fuzzy operators. However, the languages are not the same. In Probabilistic Similarity Logic, atomic sentences are of the form called {\sl similarity statements}, $A\stackrel{s}{=}B$, where $s$ is some similarity measure, and $A$, $B$ are entities or sets that can even be represented in an object-oriented syntax.
On the other hand, atomic sentences of the weighted logic is a fuzzy atom, same as in PSL. 
As we show below it is easy to view the weighted fuzzy logic as a generalization of PSL, and it serves as a convenient intermediate language to relate PSL and MLN.\footnote{%
Although PSL and Probabilistic Similarity Logic seem to be closely related, the formal relationship between them has not been discussed in the literature to the best of our knowledge. }


\cbla



\subsection{Relation to PSL}


The following lemma tells us how the notions of distance to satisfaction in PSL and in the weighted fuzzy logic are related.

\begin{lemma}\label{lem:distance}
For any rule $R$ of the form \eqref{eq:psl-rule} and any interpretation $I$,
\[ d_R(I) = 1-\u_I(R).
\]
\end{lemma}
\begin{proof}
{\small
\begin{align*}
 & 1- \u_I(R) \\
          &= max\{0,\ 1- \u_I(R)\} \\
          &= max\{0,\ 1- \u_I(a\ar_l b_1\fand_l\dots\fand_l b_n)\} \\
          &= max\{0,\ 1- min\{1 - \u_I(b_1\fand_l\cdots\fand_l b_n) +  \u_I(a),\ 1\}\}\\
          &= max\{0,\ \u_I(b_1\fand_l\cdots\fand_l b_n)- \u_I(a)\}) \\
          &= d_R(I).  \qed
\end{align*}
}
\end{proof}

In Lemma~\ref{lem:distance}, it is essential that rules~\eqref{eq:psl-rule} use Lukasiewicz fuzzy operators. The lemma does not hold with an arbitrary selection of fuzzy operators as the following example indicates. 

\begin{example}\label{ex:distance}
Consider G\"odel t-norm $\otimes_m$ and its residual implicator $\rar_r$. 
Let $R$ be $q\ar_r p$ and $I$ an interpretation $\{(p, 0.6), (q, 0.4)\}$.  $d_R(I)$ is $0.2$, while $1-\u_I(p\rar_r q)$ is $1-0.4 = 0.6$.
\end{example}


It follows from Lemma~\ref{lem:distance} that PSL can be easily viewed as a special case of the weighted fuzzy logic. 

\begin{thm}\label{thm:generalPSL2PSL}
Given any PSL program~$\Pi$ and any fuzzy interpretation $I$, the definition of $f_\Pi(I)$ when $\Pi$ is viewed as the weighted fuzzy logic coincides with the definition of $f_\Pi(I)$ when $\Pi$ is viewed as a PSL program.
\end{thm}

\begin{proof}
Immediate from Lemma~\ref{lem:distance}. \qed
\end{proof}

Due to this theorem,  we will call the weighted fuzzy logic also as {\em generalized PSL (GPSL)}. 

\BOC
It is known that inference in PSL is $\#$P-hard. Since PSL is a special case of GPSL, it is clear that the inference in GPSL is $\#$P-hard as well.
Obviously, the advantage of PSL over GPSL is that by restricting the fuzzy operators to be Lukasiewicz t-norm and its derived operators,  inference in PSL becomes a convex optimization problem in continuous space \cite{bach15hinge}.
\EOC

Viewing PSL as a special case of the weighted fuzzy logic allows us to apply the mathematical results known from fuzzy logic to the context of PSL. Here is one example, which tells us that the different versions of PSL defined in  \cite{kimmig12ashort} and \cite{bach15hinge} are equivalent despite the different syntax adopted in each of them.
To be precise, PSL in \cite{bach15hinge} is defined for clausal form only, such as \eqref{clause} below, while in~\cite{kimmig12ashort} it is defined for rule form \eqref{eq:psl-rule} only.


%
%
%
%
%
When $L$ is either an atom $A$ or $\neg_s A$, by $\o{L}$ we denote a literal complementary to $L$, i.e., $\o{L} = \neg_s A$ if $L$ is $A$, and  $\o{L} = A$ if $L = \neg_s A$. 
The following equivalences are known from fuzzy logic.

\begin{lemma} \label{lem:equiv} For any formulas $F$ and $G$, and any literals $L_i$ ($1\le i\le n$), 
\begin{itemize}
\item[(a)] $F\rar_l G$ is equivalent to $\neg_s F\for_l G$. 
\item[(b)] $\neg_s (L_1\fand_l\cdots\fand_l L_n)$ is equivalent to 
       $(\o{L_1}\for_l\cdots\for_l \o{L_n})$.
\end{itemize}
\end{lemma}


The following lemma tells us that the clausal form using Lukasiewicz t-conorm can be written in many different forms.
\begin{lemma}
For any literals $L_i$ ($1\le i\le n$), 
\beq
   L_1\for_l\cdots\for_l L_m \for_l L_{m+1}\for_l\cdots \for_l L_n 
\eeq{clause}
is equivalent to
\[
  \o{L_1}\fand_l\cdots\fand_l\o{L_m}\rar_l L_{m+1}\for_l\cdots \for_l L_n 
\]
where $n\ge m\ge 0$. 
\end{lemma}
\begin{proof}
By Lemma~\ref{lem:equiv}~(a), formula~\eqref{clause} is equivalent to 
\[
   \neg_s (L_1\for_l\cdots\for_l L_m)\rar_l L_{m+1}\for_l\cdots \for_l L_n 
\]
and by Lemma~\ref{lem:equiv}~(b), the latter is equivalent to 
\[
 \o{L_1}\fand_l\cdots\fand_l\o{L_m}\rar_l L_{m+1}\for_l\cdots \for_l L_n  . \qed
\]
\end{proof}

It follows from Lemma~\ref{lem:distance} that the probability density of an interpretation does not change when the formula is replaced with another equivalent formula. This tells us that PSL rules of the form~\eqref{eq:psl-rule} can be rewritten as any other equivalent formulas. For instance, PSL rule 
\beq
\ba {rrclcr}
 w:& \ \ \    a &\ \ \ar_l\ \ & b\fand_l c & \ \ \ \ \   &  \string^ 1 
\ea
\eeq{kimmig}
can be equivalently rewritten as any of the following ones.
\beq
\ba {rrclcr}
 w:& \neg_s b &\ \ar_l \ & \neg_s a\fand_l c   &  \string^ 1,  \\
 w:& \neg_s c &\ar_l& \neg_s a\fand_l b   &  \string^ 1,   \\
 w:& a\for_l \neg_s b &\ar_l&  c   &  \string^ 1,  \\
 w:& a\for_l \neg_s c &\ar_l&  b   &  \string^ 1,  \\
 w:& \neg_s b\for_l \neg_s c &\ar_l&  \neg_s a   &  \string^ 1,  \\
 w:& \ \ a \for_l \neg_s b \for_l \neg_s c  &   &  &  \string^ 1,  \\ 
 w:&  0 & \ar_l & \neg_s a \fand_l b \fand_l c  \ \  &  \string^ 1 . 
\ea
\eeq{barch}

\BOC
\begin{lemma}
\[ a_1\ar_l b_1\fand_l\cdots\fand_l b_n\]
is equivalent to 
\[ a_1\for_l \o{b_1}\for_l\cdots\for_l \o{b_n}.\]
\end{lemma}

\begin{proof}
The proof is immediate from the fact that De Morgan's Laws holds in the fuzzy logic and the definition of $\rar_l$. \qed
\end{proof}
\EOC

As noted above, the syntax of PSL in \cite{bach15hinge} is clausal form only, such as the second to the last formula in~\eqref{barch}, while the syntax of PSL in \cite{kimmig12ashort} is rule form such as \eqref{kimmig}.
The result above tells us that the definitions of PSL defined in \cite{kimmig12ashort} and \cite{bach15hinge} are equivalent despite the different syntax adopted there. 

On the other hand, similar rewriting using other t-norms and their derived operators may not necessarily yield an equivalent formula because not every selection of fuzzy operators satisfy Lemma~\ref{lem:equiv} even if they are generalizations of the corresponding Boolean connectives.

\begin{example}\label{ex:equiv-counter}
Consider again G\"odel t-norm $\otimes_m$ and its residual implicator $\rar_r$. The negation $\neg_m$ induced from $\neg_m x = x\rar_r 0$ is 
\[
\neg_m x =\begin{cases}
1 & \text{if $x=0$}\\
0 & \text{if $x>0$.}
\end{cases}
\]
For the interpretation $I=\{(p, 0.4), (q, 0.5)\}$, we have $\u_I(\neg_m p \for_m q)=0\oplus_m 0.5=0.5$, but $\u_I(p\rar_r q)=1$. In other words, $\neg_m p\for_m q$ is not equivalent to $p\rar_r q$.
\end{example}

In the literature on PSL \cite{kimmig12ashort,bach15hinge}, the selection of Lukasiewicz t-norm is motivated by the computational efficiency gained by reducing MAP inferences to convex optimization problems.
This section presents yet another justification of Lukasiewicz t-norm in PSL from the logical perspective.



\section{GPSL : MLN = Fuzzy Logic : Boolean Logic}

Like fuzzy logic is a many-valued extension of Boolean logic, we may view GPSL as a many-valued extension of MLN.


For any classical propositional formula $F$, let $F^\mi{fuzzy}$ be the fuzzy formula obtained from $F$ by replacing $\bot$ with ${0}$,  $\top$ with ${1}$, $\neg$ with any fuzzy negation symbol, $\land$ with any fuzzy conjunction symbol, $\lor$ with any fuzzy disjunction symbol, and $\rar$ with any fuzzy implication symbol. 

For any GPSL program $\Pi$, by $TW_{\Pi}$ (``total weight'') we denote
\[
exp\left(\sum_{\langle w: F \string^ k\rangle \in \Pi}w\right).
\]

For any MLN $\L$, let $\Pi_{\L}$ be the GPSL program obtained from $\L$ by replacing each weighted formula $w: F$ in $\L$ with $w: F^\mi{fuzzy} \string^ k$, where $k$ is either $1$ or $2$.
The following theorem tells us that, for any Boolean interpretation $I$, its weight under MLN $\L$ is proportional to the unnormalized probability density under the GPSL program $\Pi_\L$. 

\begin{thm}\label{thm:w-f}
For any MLN $\L$ and any Boolean interpretation~$I$, 
\[
  W_{\L}(I) = TW_{\Pi_\L}\cdot \hat{f}_{\Pi_\L}(I).
\]
\end{thm}

\begin{proof}
{\small 
\begin{align*}
 W_\L(I) &=  exp\Bigg(\sum_{\langle w, F\rangle  \in \L_I} w\Bigg) 
 = exp\Bigg(\sum_{\langle w, F\rangle\in \L}w - \sum_{\langle w, F \rangle \in \L\setminus \L_I} w \Bigg)\\
 &= exp\Bigg( \sum_{\langle w, F\rangle \in \L} w\Bigg)\cdot exp\Bigg(- \sum_{\langle w, F\rangle\in \L\setminus \L_I}w\Bigg)\\
 &= exp\Bigg( \sum_{\langle w, F\rangle \in \L} w\Bigg)  \\
   & \quad 
\times exp\left(- 
     \Bigg(\sum_{\langle w, F\rangle\in {\L\setminus \L_I}}(w\cdot 1) 
         + \sum_{\langle w, F\rangle\in \L_I} (w\cdot 0) \Bigg)\right). \numberthis \label{eq:w-f1}
\end{align*}
}

Note that when $I$ is Boolean, $1-\u_I(F) = 1$ if $I\not\models F$, and $1-\u_I(F) = 0$ if $I\models F$. So \eqref{eq:w-f1} is equal to 
{
\begin{align*}
  & TW_{\Pi_\L}\cdot exp\Bigg(- \sum_{\langle w: F^\mi{fuzzy} \string^ k\rangle\in \Pi_\L}
    w\cdot (1-\u_I(F^\mi{fuzzy}))^k\Bigg)\\
 &= TW_{\Pi_{\L}}\cdot \hat{f}_{\Pi_\L}(I).  \qed
\end{align*}
}
\end{proof}


This theorem tells us that the problem of computing the weight of an interpretation in MLN can be reduced to computing the probability density of an interpretation in GPSL. 

By Theorem \ref{thm:generalPSL2PSL}, since PSL is a special case of GPSL, the following relation between PSL and MLN follows easily.

\begin{cor}\label{cor:w-f}
For any PSL program $\Pi$ and any fuzzy Boolean interpretation $I$, let $\L$ be the MLN  obtained from $\Pi$ by replacing each fuzzy operator with its Boolean counterpart.
We have
\[
  \hat{f}_{\Pi}(I) = \dfrac{W_\L(I)}{TW_\Pi}.
\]
\end{cor}




\begin{example}
Let $\Pi$ be the following PSL program 
\[
\ba {rrclcl}
 1: &\ \ \  p & \ \ \ar_l\ \  & q  & \ \ \ \   &  \string^ 1 \\
 2: & q & \ar_l & p  &  & \string^ 1
\ea
\]
and let $\L$ be the corresponding MLN as described in Corollary~\ref{cor:w-f}.

The following table shows, for each Boolean interpretation $I$, its weight according to the MLN semantics ($W_\L(I)$) is $TW_\Pi$, which is $e^3$, multiplied by its unnormalized probability density function ($\hat{f}_\Pi(I)$) according to the PSL semantics. (We identify a Boolean interpretation with the set of atoms that are true in it.)

\begin{center}
\begin{tabular}{ | c | c | c | } \hline
Interpretation ($I$) & $W_\L(I)$ & $ \hat{f}_\Pi(I)$ \\ \hline
 $\emptyset$         & $e^3$     &  $e^0$   \\ \hline
 $\{p\}$             & $e^1$     &  $e^{-2}$   \\ \hline
 $\{q\}$             & $e^2$   &  $e^{-1}$   \\ \hline
 $\{p,q\}$           & $e^3$   &  $e^0$   \\ \hline
\end{tabular}
\end{center}
\end{example}

However, MAP states in MLN and PSL can be different because most probable interpretations in PSL may be non-Boolean.

\begin{example}
Consider the PSL program:
\beq
\ba {rrclclr}
 1:    &\ \ \ \ \   p      &\ \  \ar_l\ \  & \neg_s p & \ \ \ \  & \string^ 1 \\
 1:    & \neg_s p  &       &        &  & \string^ 1 
\ea
\eeq{psl1}
and the corresponding MLN: 
\beq
\ba {rrclclr}
 1:    &\ \ \ \   p      &\ \  \ar\ \  & \neg_s p & \ \ \ \   & ~~~~\\
 1:    & \neg_s p  
\ea
\eeq{mln1}

The most probable Boolean interpretations for MLN \eqref{mln1} are $I_1=\emptyset$ and $I_2=\left\{p\right\}$, each with weight $e^1$. 
Their unnormalized probability density for PSL program \eqref{psl1} is $e^{-1}$. However, they are not the most probable interpretations according to the PSL semantics: $I_3=\left\{(p, 0.5)\right\}$ has the largest unnormalized probability density $e^{-0.5}$.
\end{example}

The difference can be closed by adding to the weighted propositional fuzzy logic theory a ``crispifying'' rule for each atom. For any atom $p\in\sigma$, let $\mi{CRSP}(p)$ be the formula defined as
\[
  \mi{CRSP}(p) = p\for_l p \rar_l p.
\]
It is easy to check that $\u_I(p\for_l p \rar_l p) = 1$ iff $\u_I(p)$ is either $0$ or $1$. Note that although this formula uses Lukasiewicz operators, it is not expressible in PSL because $\for_l$ occurs in the body of the rule. 

For any MLN $\L$ of signature $\sigma$, let $\Pi_{\L}$ be the GPSL program obtained from $\L$ by replacing each weighted formula $w: F$ in $\L$ with $w: F^\mi{fuzzy} \string^ k$ where $k$ could be either $1$ or $2$. 
Let ${\rm CR}$ be the GPSL program
\[  
   \{\langle \alpha: \mi{CRSP}(p)\string ^ 1\rangle \mid p\in\sigma \rangle \}.
\]

The following theorem tells us that the most probable interpretations of MLN $\L$ coincides with the most probable interpretations of GPSL program $\Pi_\L\cup {\rm CR}$.
\begin{thm}\label{thm:MLN2GPSL2}
For any MLN $\L$, when $\alpha\to\infty$, 
\[
  {\rm argmax}_{I}(W_\L(I)) = {\rm argmax}_{J}(\hat{f}_{\Pi_\L\cup{\rm CR}}(J))
\]
where $I$ ranges over all Boolean interpretations and $J$ ranges over all fuzzy interpretations.
\end{thm}


\begin{proof}
We first show that, when $\alpha\to\infty$, for any Boolean interpretation $I$ and any non-Boolean interpretation $J$, we have $\hat{f}_{\Pi_\L\cup{\rm CR}}(J)<\hat{f}_{\Pi_\L\cup{\rm CR}}(I)$, which implies that no non-Boolean interpretation can be the most probable interpretations. 


First, for any non-Boolean interpretation $J$, let 
\[
  W_{\Pi_\L}(J) = -\sum_{\langle w: F \string^ k\rangle \in \Pi_\L}\bigg(w\cdot (1-\u_J(F))^k\bigg)
\]
and
\[
 W_{\rm CR}(J) = - \sum_{\langle w: F \string^ 1\rangle \in {\rm CR}} 
          \bigg(\alpha\cdot (1-\u_J(F))\bigg).
\]
Then 
\begin{align*}
\hat{f}_{\Pi_\L\cup{\rm CR}}(J) 
 &= exp(W_{\Pi_\L}(J)+W_{\rm CR}(J))\\
 &\le exp(W_{\rm CR}(J)).
\end{align*}
Since $J$ is not Boolean, there is at least one weighted formula $\alpha: p\for_l p\rar_l p\ \string^ 1\in {\rm CR}$ that is not satisfied by $J$, so that
\begin{align*}
 \hat{f}_{\Pi_\L\cup{\rm CR}}(J) \le 
                exp\big(-\alpha\cdot \big(1-\u_J(p\for_l p\rar_l p)\big)\big)
\end{align*}
where $1-\u_J(p\for_l p\rar_l p)>0$. 
Notice that 
\[
\ba l 
  \lim\limits_{\alpha\to\infty}\hat{f}_{\Pi_\L\cup{\rm CR}}(J) \leq \\
  \qquad\qquad \lim\limits_{\alpha\to\infty} exp\big(-\alpha\cdot \big(1-\u_J(p\for_l p\rar_l p)\big)\big) = 0.
\ea
\]

On the other hand, for any Boolean interpretation $I$, 
\begin{align*}
\hat{f}_{\Pi_\L\cup{\rm CR}}(I) &= exp(W_{\Pi_\L}(I)+W_{\rm CR}(I)) = exp(W_{\Pi_\L}(I)).
\end{align*}

Since $exp(W_{\Pi_\L}(I))$ does not contain $\alpha$, we have
\[
  \lim_{\alpha\to\infty}\hat{f}_{\Pi_\L\cup{\rm CR}}(I) > 0.
\]
Thus we have $\hat{f}_{\Pi_\L\cup{\rm CR}}(J)<\hat{f}_{\Pi_\L\cup{\rm CR}}(I)$ when $\alpha\to\infty$.

It follows that any fuzzy interpretation $K$ that satisfies  ${\rm argmax}_{J}(P_{\Pi_\L\cup{\rm CR}}(J))=K$ must be Boolean. By Theorem~\ref{thm:w-f}, for any Boolean interpretation $I$, we have
\[
\hat{f}_{\Pi_\L\cup{\rm CR}}(I) = \frac{exp(|\sigma|\cdot\alpha)}
                     {TW_{\Pi_\L\cup{\rm CR}}}    \times W_\L(I).
\]
Since 
$\dfrac{exp(|\sigma|\cdot\alpha)}
                     {TW_{\Pi_\L\cup{\rm CR}}}$
is constant for all interpretations, $\hat{f}_{\Pi_\L\cup{\rm CR}}(I)\propto W_{\L}(I)$. It follows that ${\rm argmax}_{I}(W_\L(I)) = {\rm argmax}_{J}(\hat{f}_{\Pi_\L\cup{\rm CR}}(J))$.
\qed
\end{proof}

\begin{example}
Consider the GPSL program:
\[ 
\ba {rrclclr}
 1:    &  p      & \ar_l & \neg_s p &  & \string^ 1 \\
 1:    & \neg_s p  &       &        &  & \string^ 1 \\ 
\alpha: & p      & \ar_l & p\for_l p & & \string^ 1. 
\ea
\]
When $\alpha\to\infty$ the most probable fuzzy interpretations are Boolean, and they are the same as the most probable interpretations for the MLN \eqref{mln1}.
\end{example}






It is known that the MAP problem in PSL can be solved in polynomial time \cite{brocheler10probabilistic}, while the same problem in MLN is $\#$P-hard. The reduction from MLN to GPSL in Theorem~\ref{thm:MLN2GPSL2} tells us that the MAP problem in GPSL is $\#$P-hard as well.
This implies that GPSL is strictly more expressive than PSL even when we restrict attention to Lukasiewicz operators.

Related to Theorem~\ref{thm:MLN2GPSL2}, relation between discrete and soft MAP states was also studied in \cite{bach15unifying,bach15hinge}, but from a different, computational perspective. There, inference on discrete MAP states is viewed as an instance of MAX SAT problems, and then approximated by relaxation to linear programming with rounding guarantee of solutions. The result indirectly tells us how MAP states in PSL are related to MAP states in MLN, but this is different from Theorem~\ref{thm:MLN2GPSL2}, which completely closes the semantic gap between them via crispifying rules.


\section{Conclusion}


In this note, we studied the two well-known formalisms in statistical relational learning from a logical perspective. 
Viewing PSL in terms of the weighted fuzzy logic gives us some useful insights known from fuzzy logic. Besides the reducibility to convex optimization problems,  the restriction to the Lukasiewicz fuzzy operators in clausal form allows intuitive equivalent transformations resembling those from Boolean logic. On the other hand, it prohibits us from using some other intuitive fuzzy operators.

In our previous work \cite{lee14stable,lee16fuzzy} we used fuzzy answer set programs to describe temporal projection in dynamic domains, where we had to use G\"odel t-norm as well as Lukasiewicz t-norm.\footnote{The main example was how the trust degree between people changes over time.} 
{\cblu There, G\"odel t-norm is necessary in expressing the commonsense law of inertia. For example, 
\beq
    \mi{Trust}(a,b,t)\fand_m \neg_{s} \neg_{s}\mi{Trust}(a,b,t\!+\!1)
          \rar_r\, \mi{Trust}(a,b,t\!+\!1),  \\
\eeq{trust}
expresses that the degree that $a$ trusts $b$ at time $t\!+\!1$ is equal to the degree at time $t$ if it can be assumed without contradicting any of the facts that can be derived. \footnote{We refer the reader to \cite{lee16fuzzy} for the precise semantics of this language.}
The fuzzy conjunction $\fand$ used here needs to satisfy that $\fand(x,y)$ is equal to either $x$ or $y$ (otherwise the trust degree at next time step would change for no reason). Obviously Lukasiewicz t-norm does not satisfy the requirement: $\fand_l(x,y)<x$ when $y<1$. In other words, if we replace $\fand_m$ with $\fand_l$, the trust degree at next time drops for no reason, which is unintuitive.
The restriction to Lukasiewicz t-norm in PSL accounts for the difficulty in directly applying PSL to temporal reasoning problems like the above example. 
Indeed, most work on PSL has been limited to static domains. 


{\cblu Since computing marginal probabilities in MLN can be reduced to computing marginal probabilities in GPSL as indicated by Theorem \ref{thm:w-f}, computing marginal probabilities in GPSL is at least $\#$P-hard. However, a sampling method could be used for such an inference. A naive sampling method is outlined below: suppose we are approximating the probability that the truth value of formula $F$ falls into $(l, u)$ for some $0 \leq l\leq u\leq 1$ (denoted as $P(l\leq F \leq u)$).
\begin{enumerate}
\item Generate $N$ interpretations at random;
\item For each of the $N$ interpretations, compute its probability density;
\item Approximate $P(l\leq F \leq u)$ by $\frac{X}{N}$, where $X$ is the number of interpretations $I$ that satisfies $l\leq \u_I(F) \leq u$ among the $N$ interpretations.
\end{enumerate}
It can be shown that $\frac{X}{N}$ is the estimation of $P(l\leq F \leq u)$ that maximizes the likelihood of the $N$ samples.
}

The way that MLN extends propositional logic is similar to the way that PSL extends a restricted version of fuzzy propositional logic. GPSL is simply taking the fuzzy propositional logic in full generality and applying the log-linear weight scheme.
In our recent work \cite{lee16weighted}, we adopted the similar weight scheme to answer set programs in order to overcome the deterministic nature of the stable model semantics providing ways to resolve inconsistencies in answer set programs, to rank stable models, to associate probability to stable models, and to apply statistical inference to computing weighted stable models.
Perhaps this indicates the universality of the log-linear weight scheme first adopted in MLN, which provides a uniform method to turn the crisp logic (be it fuzzy logic, propositional logic, or answer set programs) ``soft.'' 


\medskip\noindent
{\bf Acknowledgements}\ \ 
We are grateful to Michael Bartholomew and the anonymous referees for their useful comments.
This work was partially supported by the National Science Foundation under Grants IIS-1319794, IIS-1526301, and a gift funding from Robert Bosch LLC.

\bibliographystyle{aaai}

\end{document}